\PassOptionsToPackage{numbers,compress}{natbib}
\documentclass{article}

\usepackage[preprint]{neurips_2026}

\usepackage[utf8]{inputenc}
\usepackage[T1]{fontenc}
\usepackage{hyperref}
\hypersetup{hidelinks}
\usepackage{url}
\usepackage{booktabs}
\usepackage{amsfonts}
\usepackage{amsmath,amssymb,amsthm,mathtools,bm}
\usepackage{nicefrac}
\usepackage{microtype}
\usepackage[table]{xcolor}
\usepackage{graphicx}
\usepackage{subcaption}
\usepackage{array}
\usepackage{multirow}
\usepackage{tabularx}
\usepackage{algorithm}
\usepackage{algpseudocode}
\usepackage{enumitem}
\usepackage{makecell}
\usepackage{cleveref}
\DeclareMathOperator*{\minimize}{min}

\graphicspath{{Figures/}}
\allowdisplaybreaks

\newtheorem{theorem}{Theorem}
\newtheorem{proposition}{Proposition}
\newtheorem{remark}{Remark}

\newcommand{\method}{KLA}
\newcommand{\R}{\mathbb{R}}
\newcommand{\trans}{^{\mathsf{T}}}
\newcommand{\diag}{\operatorname{Diag}}
\newcommand{\Tr}{\operatorname{tr}}
\newcommand{\Loss}{\mathcal{L}}
\newcommand{\Aset}{\mathcal{A}}
\newcommand{\wt}[1]{\widetilde{#1}}
\newcommand{\norm}[1]{\lVert #1 \rVert_2}
\newcommand{\fro}[1]{\lVert #1 \rVert_F}

\newcommand{\eps}{\varepsilon}
\newcommand{\od}{\odot}
\newcommand{\I}{I}
\newcommand{\rowhi}{\rowcolor{gray!10}}
\newcommand{\kmark}[1]{\textcolor{blue}{#1}}

\title{Kaczmarz Linear Attention}
\author{%
  Jiaxuan Zou \\
  School of Mathematics and Statistics \\
  Xi'an Jiaotong University \\
  Xi'an, China \\
  \And
  Ruifeng Ren \\
  Gaoling School of Artificial Intelligence \\
  Renmin University of China \\
  Beijing, China \\
  \And
  Yong Liu\thanks{Corresponding author.} \\
  Gaoling School of Artificial Intelligence \\
  Renmin University of China \\
  Beijing, China \\
  \texttt{liuyonggsai@ruc.edu.cn}
}

\begin{document}

\maketitle


\begin{abstract}
Long-context language modeling remains central to modern sequence modeling, but the quadratic cost of Transformer attention makes scaling computationally prohibitive.
Linear recurrent models address this bottleneck by compressing the context into a fixed-size state, making the rule that forgets, writes, and edits information a central design problem.
To address state maintenance, Gated DeltaNet (GDN) combines gated state decay with delta-rule residual writes, using a learnable coefficient to balance forgetting and update magnitude.
However, this coefficient is learned empirically rather than derived from the underlying objective, which can lead to suboptimal update magnitudes.
We revisit the online-regression objective underlying GDN and, inspired by the Kaczmarz projection method, derive the key-norm-normalized dynamic step size $\beta_t = \kmark{\eta_t/(\norm{k_t}^2+\eps)}$ for residual updates.
We propose \textbf{Kaczmarz Linear Attention} (\method{}), a one-scalar modification of GDN that preserves the state shape, gates, linear recurrence, and chunkwise parallel algorithm.
At the 0.4B scale with a 1B-token budget, \method{} achieves the lowest validation perplexity among evaluated linear-time baselines (8.09 vs.\ 8.50 for GDN) and remains stable up to 65K tokens.
On controlled tasks, \method{} reaches 100\% on single-needle-in-a-haystack (S-NIAH) retrieval, improves $8\times$ multi-query associative recall (MQAR) by 7.03 points over GDN, and delivers $2.1\times$ higher decode throughput at 32K context.
These results suggest that the key-norm-normalized Kaczmarz coefficient is a first-order design axis for delta-rule sequence models: it improves accuracy, extrapolation, and decoding efficiency without changing the recurrent state or hardware kernel.
An anonymized implementation is included in the supplementary material and available at
\url{https://github.com/anonymous-kla-review/kla-anonymous}.
\end{abstract}

\section{Introduction}
\label{sec:intro}

Softmax attention gives Transformers strong associative recall, but its token-token attention matrix scales quadratically with sequence length, increasing long-context prefill latency, activation memory, and train-short/test-long evaluation cost~\citep{press2022trainshorttestlong,shaham2022scrollsstandardizedcomparisonlong}.
IO-aware kernels and sequence-parallel implementations reduce constant factors, yet they do not remove this asymptotic cost~\citep{NIPS2017_3f5ee243,bahdanau2016neuralmachinetranslationjointly,dao2022flashattentionfastmemoryefficientexact,dao2023flashattention2fasterattentionbetter,brandon2023stripedattentionfasterring,liu2023blockwiseparalleltransformerlarge}.

Linear-time sequence models seek to alleviate this bottleneck by replacing token-token attention with a fixed-size recurrent state.
As all past-token information is compressed into bounded memory, the state-update rule---which determines what the model writes, overwrites, preserves, and later retrieves---becomes the central design choice.
Different lines of work propose different update rules based on distinct design insights.
Linear attention and fast-weight models use additive key-value writes~\citep{pmlr-v119-katharopoulos20a,pmlr-v139-schlag21a}.
Gated recurrent attention adds data-dependent decay~\citep{sun2023retentivenetworksuccessortransformer,peng2023rwkvreinventingrnnstransformer,pmlr-v235-yang24ab,qin2024hgrn}.
Selective state-space or hybrid models scale related recurrence mechanisms to large language models~\citep{gu2024mamba,dao2024transformersssmsgeneralizedmodels,de2024griffinmixinggatedlinear,lieber2024jambahybridtransformermambalanguage}.
Although these models reduce quadratic attention to linear-time recurrence, their writes do not explicitly account for the key-value associations already stored in the state.
When similar keys recur, repeated writes can redundantly accumulate related information instead of editing the existing association.

Delta-rule models address this redundancy by writing the residual between the target value and the value predicted by the current state.
DeltaNet therefore overwrites a specific key direction instead of only adding or decaying information~\citep{yang2024parallelizing}.
To control the residual-write magnitude more finely, Gated DeltaNet (GDN) introduces a learned scalar that balances forgetting and updating and improves long-context retrieval~\citep{yang2025gated}.
However, this scalar is primarily an empirical design choice rather than a theoretically specified step size.
An inappropriate coefficient can over-correct large-norm keys, under-correct small-norm keys, and leave the write scale mismatched to the current key.
Because the GDN update admits an online-regression view, the coefficient can instead be derived from the per-token optimization problem.

\method{} resolves this coefficient problem by treating each delta-rule write as a Kaczmarz projection.
The key insight is simple: a key-value write imposes one linear constraint, $S_t\trans k_t = v_t$, and the closest state satisfying that constraint is reached with the Kaczmarz step size $\beta_t=\kmark{\eta_t/(\norm{k_t}^2+\eps)}$~\citep{kaczmarz1937}.
This self-normalized projection turns the gated delta rule into an exact per-token update in the unrelaxed case while preserving the GDN state shape, gates, linear recurrence, and chunkwise parallel kernel.
Thus \method{} changes the coefficient, not the architecture.

We make four contributions:
\begin{itemize}[leftmargin=1.5em,itemsep=0.3em,topsep=0.3em]
    \item \textbf{Self-normalized delta rule.} We derive \method{} from the Kaczmarz projection and show that its coefficient is simultaneously an orthogonal projection step, an exact line-search minimizer, and a normalized gradient step for the per-token loss.
    \item \textbf{Drop-in recurrent algorithm.} We show that \method{} preserves the GDN recurrence and chunkwise solver; the kernel changes only through the diagonal coefficient matrix $B=\diag(\beta_1,\ldots,\beta_C)$.
    \item \textbf{Empirical gains.} At the 0.4B scale with a 1B-token budget, \method{} obtains the lowest validation perplexity among evaluated linear-time baselines and improves synthetic retrieval and state-tracking, including single-needle-in-a-haystack (S-NIAH), multi-query associative recall (MQAR), and Stack.
    \item \textbf{Preserved efficiency.} \method{} keeps $O(n)$ time and memory complexity, matches GDN prefill latency, and improves decode throughput without a new state structure or hardware kernel.
\end{itemize}

\Cref{tab:main_result} summarizes the main language-modeling result.

\begin{table}[h]
\centering
\small
\setlength{\tabcolsep}{5pt}
\renewcommand{\arraystretch}{1.12}
\begin{tabular}{lcccccc}
\toprule
 & \textbf{KLA} & DeltaNet & GLA & GDN & Longhorn & Mamba2 \\
\midrule
Validation PPL $\downarrow$ & \textbf{8.09} & 8.26 & 8.42 & 8.50 & 8.79 & 9.41 \\
\bottomrule
\end{tabular}
\caption{Pretraining validation perplexity (PPL) at the 0.4B scale with a 1B-token budget on SlimPajama-672B~\citep{cerebras2023slimpajama}. All models use the same tokenizer, optimizer, and 4K training context.}
\label{tab:main_result}
\end{table}

\section{Background and Problem Analysis}
\label{sec:background}

Linear-time sequence models replace the token-token attention matrix with a recurrent state $S_t \in \R^{d_k \times d_v}$, so the main design choice is the write rule applied to this state. Linear attention writes each key-value pair additively, $S_t = S_{t-1}+k_t v_t\trans$, and a query reads $o_t=S_t\trans q_t$~\citep{pmlr-v119-katharopoulos20a}. Gated Linear Attention (GLA) adds data-dependent forgetting, $S_t=\diag(\alpha_t)S_{t-1}+k_t v_t\trans$, so earlier content can decay before the next write~\citep{pmlr-v235-yang24ab}. In the scalar-gated notation used below, this forget-then-write pattern is
\begin{equation*}
\wt{S}_{t-1}=\alpha_t S_{t-1},\qquad
S_t=\wt{S}_{t-1}+k_t v_t\trans .
\end{equation*}
These recurrent forms are efficient; the delta-rule family changes the write itself.

Delta-rule models make the state content-editable by writing the residual between the target value and the current state prediction. In the scalar-gated GDN template~\citep{yang2025gated}, the residual is measured and written back along the current key direction:
\begin{equation*}
S_t = \wt{S}_{t-1}+\beta_t k_t e_t\trans
= (\I-\beta_t k_t k_t\trans)\wt{S}_{t-1}+\beta_t k_t v_t\trans,
\qquad
 e_t = v_t-\wt{S}_{t-1}\trans k_t .
\end{equation*}
The rank-one correction edits the state in the active key direction, and $\beta_t$ determines the step length of this edit. After decay, token $t$ can be viewed through a data-fit loss and its regularized online-regression form:
\begin{equation*}
\Loss_t(S)=\frac{1}{2}\|S\trans k_t-v_t\|_2^2,
\qquad
\mathcal{J}_t(S;\lambda_t)=\Loss_t(S)+\frac{\lambda_t}{2}\|S-\wt{S}_{t-1}\|_F^2 .
\end{equation*}
Here $\Loss_t$ fits the current key-value association, while the regularization term keeps the new state close to the decayed memory. The residual direction $k_t e_t\trans$ is the negative gradient direction of $\Loss_t$ at $\wt{S}_{t-1}$, and the scalar coefficient controls how strongly this residual is written. In the GDN template, this coefficient is supplied by a learned gate, $\beta_t=\eta_t$; equivalently, the update magnitude is learned empirically rather than derived from the regularized objective. The next section derives this coefficient from the per-token optimization problem itself.

\section{Kaczmarz Linear Attention}
\label{sec:method}

\subsection{Kaczmarz coefficient from constrained optimization}

The regularized view above isolates the residual direction from the scalar that controls its magnitude. We now derive the scalar by treating a key-value write as a single linear constraint on the recurrent state. The Kaczmarz method is a classical row-action method for solving linear systems by repeatedly projecting the current iterate onto one constraint hyperplane at a time~\citep{kaczmarz1937}. For token $t$, we choose the state closest to the decayed state $\wt{S}_{t-1}$ that satisfies the current key-value constraint:
\begin{equation}
\minimize_{S \in \R^{d_k \times d_v}} \;\tfrac{1}{2}\fro{S - \wt{S}_{t-1}}^2 \quad \text{subject to} \quad S\trans k_t = v_t.
\label{eq:constrained}
\end{equation}
This problem asks for the minimum-norm state change that makes key $k_t$ read out value $v_t$. The constraint set $\Aset_t=\{S:S\trans k_t=v_t\}$ is an affine subspace with normal direction $k_t$, so the optimal correction has the form $S_t=\wt{S}_{t-1}+k_t c\trans$ for some $c\in\R^{d_v}$. Substituting this form into the constraint gives $\wt{S}_{t-1}\trans k_t+c\norm{k_t}^2=v_t$, hence $c=e_t/\norm{k_t}^2$, and the exact projection is
\begin{equation}
S_t = \wt{S}_{t-1} + \frac{1}{\norm{k_t}^2}\,k_t\, e_t\trans.
\label{eq:kaczmarz_exact}
\end{equation}
Thus Kaczmarz writes the residual $e_t$ with a step length normalized by the key energy. The practical \method{} update adds a learned relaxation gate and a numerical stabilizer,
\begin{equation}
S_t = \wt{S}_{t-1} + \beta_t\,k_t e_t\trans,
\qquad
\beta_t = \frac{\kmark{\eta_t}}{\kmark{\norm{k_t}^2 + \eps}}.
\label{eq:kla}
\end{equation}
When $\eta_t=1$ and $\eps=0$, \eqref{eq:kla} reduces to the exact Kaczmarz projection in \eqref{eq:kaczmarz_exact}. The relaxed update also has an online-regression interpretation:
\begin{proposition}[Exact proximal form]
\label{prop:proximal}
For any $\eta_t \in (0,1]$ and $\eps \ge 0$, define $\mu_t = \eta_t / ((1-\eta_t)\norm{k_t}^2+\eps)$. The relaxed update in~\eqref{eq:kla} is the exact minimizer of
\[
\min_S\; \frac{1}{2}\fro{S-\wt{S}_{t-1}}^2 + \frac{\mu_t}{2}\|S\trans k_t-v_t\|_2^2.
\]
\end{proposition}
Thus, the practical coefficient can be read either as a relaxed Kaczmarz step or as the closed-form solution of a proximal local regression problem.

\paragraph{Connection to GDN.}
\label{sec:scale_mismatch}
To make the coefficient-level comparison explicit, return to the data-fit loss $\Loss_t$ defined in \cref{sec:background}. GDN and \method{} use the same decayed state $\wt{S}_{t-1}$, residual $e_t$, and rank-one write direction $k_t e_t\trans$; they differ only in the scalar multiplying that write. Along the one-parameter path $S(\tau)=\wt{S}_{t-1}+\tau k_t e_t\trans$, substituting into $\Loss_t$ gives
\begin{equation}
\Loss_t(S(\tau)) = \frac{1}{2}\bigl(1-\tau\norm{k_t}^2\bigr)^2\norm{e_t}^2.
\label{eq:line_search_loss}
\end{equation}
The exact line-search coefficient is therefore $\tau^*=1/\norm{k_t}^2$. GDN uses $\beta_t=\eta_t$, so it matches this step only when the learned scalar happens to equal $1/\norm{k_t}^2$ for the current token. In contrast, \method{} uses $\beta_t=\eta_t/(\norm{k_t}^2+\eps)$ and recovers the exact minimizer when $\eta_t=1$ and $\eps=0$. The improvement is not a new write direction; it is a key-energy correction to the GDN step length.

\Cref{tab:online_lens} places this coefficient-level change alongside representative recurrent layers. The comparison separates the local objective from the update rule while keeping the recurrent state fixed.
\begin{table*}[h]
\centering
\scriptsize
\setlength{\tabcolsep}{3.5pt}
\renewcommand{\arraystretch}{1.12}
\resizebox{\textwidth}{!}{%
\begin{tabular}{lll}
\toprule
Method & Local objective $\Loss_t$ & Update rule \\
\midrule
Linear attention & $-\langle S_{t-1}\trans k_t, v_t\rangle$ & $S_t = S_{t-1} + k_t v_t\trans$ \\
RetNet / Mamba2 & $-\beta_t\langle S_{t-1}\trans k_t,v_t\rangle + \frac{1}{2}\|\sqrt{1-\alpha_t}\,S_{t-1}\|_F^2$ & $S_t = \alpha_t S_{t-1} + \beta_t k_t v_t\trans$ \\
GLA & $-\langle S_{t-1}\trans k_t,v_t\rangle + \frac{1}{2}\|\sqrt{\diag(1-\alpha_t)}\,S_{t-1}\|_F^2$ & $S_t = \diag(\alpha_t)S_{t-1} + k_t v_t\trans$ \\
Longhorn & $\frac{\beta_t}{2}\|S_{t-1}\trans k_t-v_t\|_2^2$ & $S_t = (\I-\rho_t k_t k_t\trans)S_{t-1}+\rho_t k_t v_t\trans$ \\
DeltaNet & $\frac{\beta_t}{2}\|S_{t-1}\trans k_t-v_t\|_2^2$ & $S_t = (\I-\beta_t k_t k_t\trans)S_{t-1}+\beta_t k_t v_t\trans$ \\
GDN & $\frac{\beta_t}{2}\|\wt{S}_{t-1}\trans k_t-v_t\|_2^2$ & $S_t = (\I-\beta_t k_t k_t\trans)\wt{S}_{t-1}+\beta_t k_t v_t\trans$ \\
\rowhi \textbf{KLA} & $\frac{1}{2}\|\wt{S}_{t-1}\trans k_t-v_t\|_2^2$ & $S_t = \wt{S}_{t-1}+\dfrac{\kmark{\eta_t}}{\kmark{\norm{k_t}^2+\eps}}k_t\bigl(v_t-\wt{S}_{t-1}\trans k_t\bigr)\trans$ \\
\bottomrule
\end{tabular}}
\caption{Online-learning view of recurrent updates. \method{} is the only model in the table that exactly minimizes the per-token loss in the unrelaxed case ($\eta_t=1$, $\eps=0$).}
\label{tab:online_lens}
\end{table*}

\subsection{Why the Kaczmarz coefficient is natural}
\label{sec:theory}

The constrained optimization view gives a candidate coefficient, and the following result explains why this coefficient is natural. The same scalar simultaneously satisfies three desiderata that are usually treated separately in online learning: it is an orthogonal projection, an exact line-search minimizer, and a normalized gradient step. This unified characterization justifies why the key-norm term should appear in the delta-rule coefficient rather than in an auxiliary heuristic.

\begin{theorem}[Projection, line search, and normalized SGD]
\label{thm:projection}
Assume $\eta_t = 1$ and $\eps = 0$. Let $\wt{S}_{t-1} = \alpha_t S_{t-1}$, $e_t = v_t - \wt{S}_{t-1}\trans k_t$, and $\Aset_t = \{S : S\trans k_t = v_t\}$. The update
\[
S_t = \wt{S}_{t-1} + \frac{1}{\norm{k_t}^2}\,k_t\, e_t\trans
\]
satisfies $S_t\trans k_t = v_t$. Moreover, the same update is simultaneously:
\begin{enumerate}[leftmargin=1.5em,itemsep=0.1em,topsep=0.1em,label=(\roman*)]
    \item the orthogonal projection of $\wt{S}_{t-1}$ onto $\Aset_t$ in Frobenius norm;
    \item the exact minimizer of $\Loss_t(S)$ along the direction $\wt{S}_{t-1} + \tau k_t e_t\trans$; and
    \item a normalized SGD update on $\Loss_t$ with coefficient $1/\norm{k_t}^2$.
\end{enumerate}
\end{theorem}

\begin{remark}
By the projection theorem for Hilbert spaces, $S_t$ is the unique element of $\Aset_t$ closest to $\wt{S}_{t-1}$. The Kaczmarz update makes the minimum-norm change to the state while exactly satisfying the constraint.
\end{remark}

\begin{proposition}[Residual contraction]
\label{prop:contraction}
Let $e_t = v_t - \wt{S}_{t-1}\trans k_t$ and $e_t^+ = v_t - S_t\trans k_t$. Then
\begin{equation}
e_t^+ = \Bigl(1 - \eta_t\frac{\norm{k_t}^2}{\norm{k_t}^2 + \eps}\Bigr)\,e_t.
\label{eq:contraction}
\end{equation}
For $0 < \eta_t \le 1$, the residual contracts and the per-token loss decreases monotonically: $\Loss_t(S_t) \le \Loss_t(\wt{S}_{t-1})$.
\end{proposition}
\Cref{eq:contraction} gives the stability condition for the relaxed update. The contraction factor is bounded by the gate $\eta_t$ and regularized by $\eps$, so the update cannot amplify the current residual. The Kaczmarz coefficient depends on $\wt{S}_{t-1}$ and $k_t$, not on the form of the decay operator; therefore \method{} and KDA~\citep{kimiteam2025kimilinearexpressiveefficient} are orthogonal improvements. Proofs are in \cref{app:proofs}.

\subsection{Tokenwise implementation}
\label{sec:impl}

At the token level, \method{} keeps the same decay, residual computation, and rank-one write as GDN. The implementation changes only the coefficient line:
\begin{center}
\small\ttfamily
\fcolorbox{black!20}{black!2}{\begin{minipage}[t]{0.42\textwidth}
\textbf{GDN}\par\medskip
Sbar = alpha * S\par
err  = v - Sbar\^{}T @ k\par
beta = eta\par
S    = Sbar + beta * outer(k, err)
\end{minipage}}
\hfill
\fcolorbox{black!20}{blue!3}{\begin{minipage}[t]{0.44\textwidth}
\textbf{KLA}\par\medskip
Sbar = alpha * S\par
err  = v - Sbar\^{}T @ k\par
beta = eta / (norm(k)**2 + eps)\par
S    = Sbar + beta * outer(k, err)
\end{minipage}}
\end{center}

The left panel of \cref{alg:kla_updates} gives the per-token update, and the right panel gives the matching chunkwise solver. As in GDN, $\eta_t$ is a sigmoid-activated linear projection of the input token. The query $q_t$ is $\ell_2$-normalized before the readout $o_t=S_t\trans q_t/\norm{q_t}$.


\begin{algorithm}[h]
\caption{Tokenwise and chunkwise \method{} computation}
\label{alg:kla_updates}
\small
\begin{minipage}[t]{0.46\textwidth}
\textbf{Tokenwise update (one head)}
\begin{algorithmic}[1]
\Require $S_{t-1}$, $\alpha_t$, $k_t$, $v_t$, $q_t$, $\eta_t$, $\eps$
\State $\wt{S}_{t-1} \gets \alpha_t S_{t-1}$
\State $e_t \gets v_t - \wt{S}_{t-1}\trans k_t$
\State $\beta_t \gets \eta_t/(\norm{k_t}^2+\eps)$
\State $S_t \gets \wt{S}_{t-1}+\beta_t k_t e_t\trans$
\State $o_t \gets S_t\trans q_t/\norm{q_t}$
\State \Return $o_t, S_t$
\end{algorithmic}
\end{minipage}\hfill
\begin{minipage}[t]{0.51\textwidth}
\textbf{Chunkwise kernel (chunk length $C$)}
\begin{algorithmic}[1]
\Require $S_0$, $K,V,Q$, $\{\alpha_i\}$, $\{\eta_i\}$, $\eps$
\For{$i=1$ \textbf{to} $C$}
    \State $\gamma_i \gets \alpha_i\gamma_{i-1}$; $\beta_i \gets \eta_i/(\norm{k_i}^2+\eps)$
\EndFor
\State Form $B=\diag(\beta_1,\ldots,\beta_C)$, $D_\gamma$, $A$, $A^-$
\State Solve $(\I+B(A^-\od KK\trans))U=B(V-D_\gamma K S_0)$
\State $O \gets D_\gamma Q S_0 + (A\od QK\trans)U$
\State $S_{\mathrm{out}} \gets \gamma_C S_0 + K\trans\diag(\gamma_C/\gamma_1,\ldots,1)U$
\State \Return $O, S_{\mathrm{out}}$
\end{algorithmic}
\end{minipage}
\end{algorithm}

\section{Chunkwise Parallelism}
\label{sec:chunk}

\method{} inherits the GDN chunkwise-parallel training algorithm because the Kaczmarz coefficient changes only the scalar multiplying each residual write. With $R_t=S_t\trans\in\R^{d_v\times d_k}$, the update in \eqref{eq:kla} can be written as
\begin{equation}
R_t = \alpha_t R_{t-1}(\I-\beta_t k_t k_t\trans)+\beta_t v_t k_t\trans,
\qquad \beta_t=\frac{\eta_t}{\norm{k_t}^2+\eps}.
\label{eq:transposed}
\end{equation}
This is the GDN recurrence with a different diagonal coefficient. A comparison with representative recurrent and chunkwise-parallel forms is deferred to \cref{tab:recurrent_parallel}; the within-chunk triangular system is derived in \cref{app:chunk}.

For a chunk of length $C$, define stacked matrices $K,V,Q\in\R^{C\times d}$, auxiliary matrix $U\in\R^{C\times d_v}$ with rows $u_j\trans$, $B=\diag(\beta_1,\ldots,\beta_C)$, $D_\gamma=\diag(\gamma_1,\ldots,\gamma_C)$, and $\gamma_i=\prod_{r=1}^i\alpha_r$. Let $A_{ij}=\gamma_i/\gamma_j$ for $j\le i$ and zero otherwise, and let $A^-$ be the strict lower-triangular part of $A$. These are the same objects used by the GDN solver; only $B$ is formed from the Kaczmarz coefficient.

\begin{proposition}[Chunkwise solve]
\label{prop:chunk}
The auxiliary vectors $U$ satisfy the lower-triangular linear system
\begin{equation}
\bigl(\I + B(A^- \od KK\trans)\bigr)\,U = B\bigl(V - D_\gamma K S_0\bigr).
\label{eq:chunk_system}
\end{equation}
The left-hand matrix has unit diagonal and can be solved by forward substitution or the GDN chunkwise kernel.
\end{proposition}

\begin{proposition}[Chunk outputs and outgoing state]
\label{prop:chunk_output}
The chunk outputs and final state are
\begin{align}
O &= D_\gamma Q S_0 + (A \od QK\trans)\,U, \label{eq:chunk_out} \\
S_{\mathrm{out}} &= \gamma_C S_0 + K\trans \diag\!\bigl(\gamma_C/\gamma_1,\ldots,1\bigr)\,U. \label{eq:chunk_state}
\end{align}
\end{proposition}
\Cref{prop:chunk,prop:chunk_output} are proved in \cref{app:chunk}. They show that \method{} uses the same lower-triangular solve and the same outgoing-state formula as GDN once $B$ is populated with $\eta_i/(\norm{k_i}^2+\eps)$. \Cref{alg:kla_updates} summarizes the tokenwise and chunkwise computations side by side, avoiding two separate algorithm floats.

\section{Experiments}
\label{sec:experiments}

\subsection{Setup}

\paragraph{Pretraining.}
We evaluate whether the Kaczmarz coefficient improves language modeling under fixed compute and fixed architecture.
All models have 0.4B parameters and train on 1B tokens from SlimPajama-672B~\citep{cerebras2023slimpajama} with a 4K context, following the GDN pipeline~\citep{yang2025gated} and standard language-model scaling practice~\citep{hoffmann2022trainingcomputeoptimallargelanguage,touvron2023llamaopenefficientfoundation}.
All models use the same tokenizer, AdamW optimizer, and cosine schedule.

\paragraph{Baselines.}
We compare \method{} with GDN~\citep{yang2025gated}, DeltaNet~\citep{yang2024parallelizing}, GLA~\citep{pmlr-v235-yang24ab}, Longhorn~\citep{liu2024longhornstatespacemodels}, and Mamba2~\citep{dao2024transformersssmsgeneralizedmodels} for pretraining perplexity.
For long-context, synthetic, and efficiency evaluations, we use \method{}, GDN~\citep{yang2025gated}, and Mamba2~\citep{dao2024transformersssmsgeneralizedmodels} as a consistent three-way comparison.
All baselines use the same architecture, parameter count, and training setup.

\paragraph{Evaluation.}
We evaluate three properties: language modeling, controlled state use, and runtime efficiency.
Metrics include validation perplexity from 1K to 65K context without long-context fine-tuning, accuracy on four synthetic tasks, and prefill plus decode metrics from the FLA library~\citep{yang2024fla}.
The synthetic tasks are MQAR (multi-query associative recall), S-NIAH (single needle-in-a-haystack), Palindrome (structured reversal), and Stack (push-pop state tracking).

\subsection{Language Modeling}
\label{sec:lm_results}

The first question is whether replacing GDN's learned scalar with the Kaczmarz coefficient lowers perplexity under the same model size, data, optimizer, and context length.

\paragraph{Pretraining perplexity.}
We hypothesize that a coefficient that solves the local key-value regression problem should improve next-token prediction under the same compute budget.
\Cref{tab:main_result} reports that \method{} reaches 8.09 PPL, outperforming DeltaNet (8.26), GLA (8.42), GDN (8.50), Longhorn (8.79), and Mamba2 (9.41).
The 0.41 PPL improvement over GDN supports the hypothesis that the learned scalar in GDN leaves measurable loss on the table.
\Cref{fig:pretrain} shows the same effect in validation perplexity during pretraining: \method{} converges to a lower final validation PPL than the baselines.

\begin{figure}[h]
\centering
\includegraphics[width=0.72\linewidth]{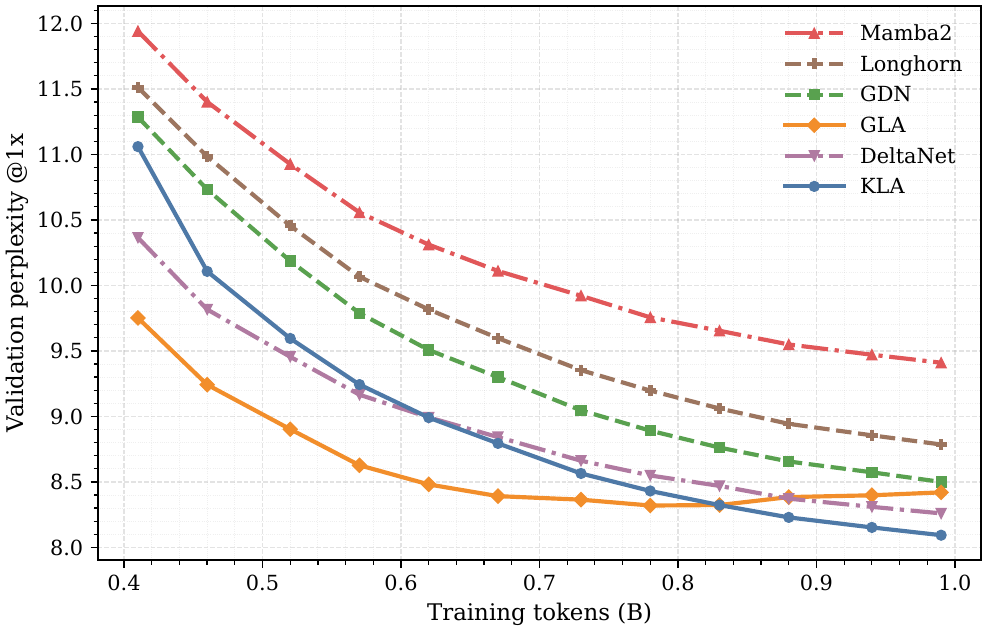}
\caption{Validation perplexity curves during pretraining. \method{} converges to a lower final validation PPL than all baselines.}
\label{fig:pretrain}
\end{figure}

\paragraph{Long-context perplexity.}
We also test whether the coefficient correction remains useful outside the 4K training context.
\Cref{tab:longcontext} reports validation perplexity from 1K to 65K tokens without long-context fine-tuning.
\method{} achieves the lowest perplexity at every tested length, with a 3.7--4.1 PPL advantage over GDN.
This persistent gap supports the interpretation that key-norm-normalized writes reduce state corruption beyond the training length, rather than only improving optimization at 4K tokens.
The tabulated values make the per-length gap explicit without relying on a trend plot.

\begin{table}[h]
\centering
\small
\setlength{\tabcolsep}{5pt}
\renewcommand{\arraystretch}{1.1}
\begin{tabular}{lccccccc}
\toprule
Model & 1K & 2K & 4K & 8K & 16K & 32K & 65K \\
\midrule
\rowhi \textbf{KLA} & \textbf{39.1} & \textbf{39.8} & \textbf{40.0} & \textbf{43.6} & \textbf{45.3} & \textbf{44.5} & \textbf{43.6} \\
GDN    & 42.9 & 43.5 & 43.8 & 47.4 & 49.4 & 48.4 & 47.5 \\
Mamba2 & 49.9 & 50.9 & 52.2 & 56.4 & 58.7 & 57.3 & 56.7 \\
\bottomrule
\end{tabular}
\caption{Long-context validation perplexity ($\downarrow$) by context length, measured in tokens. \method{} achieves the lowest perplexity at every length without long-context fine-tuning.}
\label{tab:longcontext}
\end{table}

\subsection{Synthetic Tasks: Controlled State Use}
\label{sec:synthetic_results}

We use four synthetic tasks to isolate how each model writes to and reads from its recurrent state. All tasks follow the GDN configurations~\citep{yang2025gated}: MQAR is trained at length 256 and evaluated up to $8\times$ length, S-NIAH is evaluated from 1K to 8K context, and Palindrome and Stack are evaluated by sequence length. \Cref{tab:synthetic_extrap} reports length extrapolation, and \cref{fig:synthetic} reports training convergence.

\paragraph{Palindrome.}
Palindrome requires the model to reproduce a random token sequence in reverse order, testing ordered sequence storage rather than key-value association:
\[
\begin{array}{rccccccccc}
\textbf{Input} & O & G & R & S & U & N & E & \langle\mathrm{sep}\rangle & E\;N\;U\;S\;R\;G\;O \\
\textbf{Output} & \phi & \phi & \phi & \phi & \phi & \phi & \phi & \phi & N\;U\;S\;R\;G\;O\;\phi
\end{array}
\]

\paragraph{Multi-Query Associative Recall (MQAR).}
MQAR tests whether a model can retrieve values associated with multiple query keys appearing at different positions in the context. For example, after observing pairs such as $B\mapsto 0$ and $G\mapsto 5$, the model must output the corresponding values when queried:
\[
\begin{array}{rccccccccccccccc}
\textbf{Input} & A & 1 & C & 3 & B & 0 & M & 8 & G & 5 & E & 4 & \langle\mathrm{sep}\rangle & B & G \\
\textbf{Output} & \phi & \phi & \phi & \phi & \phi & \phi & \phi & \phi & \phi & \phi & \phi & \phi & \phi & 0 & 5
\end{array}
\]

\paragraph{S-NIAH.}
Single-needle-in-a-haystack evaluates whether the model can write and retrieve one salient key-value association among long distractor contexts.

\paragraph{Stack.}
Stack simulates LIFO push-pop state tracking. The input contains push operations that insert elements into one of several stacks and pop operations that require the model to output the most recently pushed element for the queried stack. The task stresses repeated local state edits rather than a single retrieval.

\begin{table*}[h]
\centering
\setlength{\tabcolsep}{3pt}
\renewcommand{\arraystretch}{1.08}
\begin{minipage}[t]{0.48\textwidth}
\centering
\textbf{MQAR validation accuracy (\%)}\\[-1pt]
\begin{tabular}{lcccc}
\toprule
Model & 256 & 512 & 1024 & 2048 \\
\midrule
GDN    & 98.74 & 98.36 & 94.31 & \underline{66.81} \\
KLA    & \underline{99.00} & \underline{98.56} & \underline{95.64} & \textbf{73.84} \\
Mamba2 & \textbf{99.08} & \textbf{99.02} & \textbf{96.54} & 64.29 \\
\bottomrule
\end{tabular}
\end{minipage}\hfill
\begin{minipage}[t]{0.48\textwidth}
\centering
\textbf{S-NIAH exact-match accuracy (\%)}\\[-1pt]
\begin{tabular}{lcccc}
\toprule
Model & 1K & 2K & 4K & 8K \\
\midrule
GDN    & \underline{98.85} & \underline{98.25} & \underline{98.05} & \underline{98.50} \\
KLA    & \textbf{100.00} & \textbf{100.00} & \textbf{100.00} & \textbf{100.00} \\
Mamba2 & 0.10 & 0.00 & 0.00 & 0.05 \\
\bottomrule
\end{tabular}
\end{minipage}

\vspace{0.45em}

\begin{minipage}[t]{0.48\textwidth}
\centering
\textbf{Palindrome accuracy (\%)}\\[-1pt]
\begin{tabular}{lcccc}
\toprule
Model & 256 & 512 & 1024 & 2048 \\
\midrule
GDN    & \underline{99.94} & \textbf{99.81} & \underline{98.24} & \textbf{80.74} \\
KLA    & \textbf{99.96} & \underline{99.50} & \textbf{98.33} & \underline{79.96} \\
Mamba2 & 3.20 & 17.92 & 88.22 & 0.86 \\
\bottomrule
\end{tabular}
\end{minipage}\hfill
\begin{minipage}[t]{0.48\textwidth}
\centering
\textbf{Stack accuracy (\%)}\\[-1pt]
\begin{tabular}{lcccc}
\toprule
Model & 256 & 512 & 1024 & 2048 \\
\midrule
GDN    & 99.93 & 99.73 & 98.94 & \underline{95.64} \\
KLA    & \underline{99.97} & \textbf{99.96} & \underline{99.25} & \textbf{96.94} \\
Mamba2 & \textbf{99.98} & \underline{99.94} & \textbf{99.54} & 93.82 \\
\bottomrule
\end{tabular}
\end{minipage}
\caption{Length extrapolation results for synthetic tasks. All values are accuracies (\%). Bold and underline mark the best and second-best values in each column. MQAR is trained at length 256; S-NIAH is evaluated from 1K to 8K context; Palindrome and Stack are evaluated by sequence length.}
\label{tab:synthetic_extrap}
\end{table*}

\begin{figure*}[t]
\centering
\includegraphics[width=0.92\textwidth]{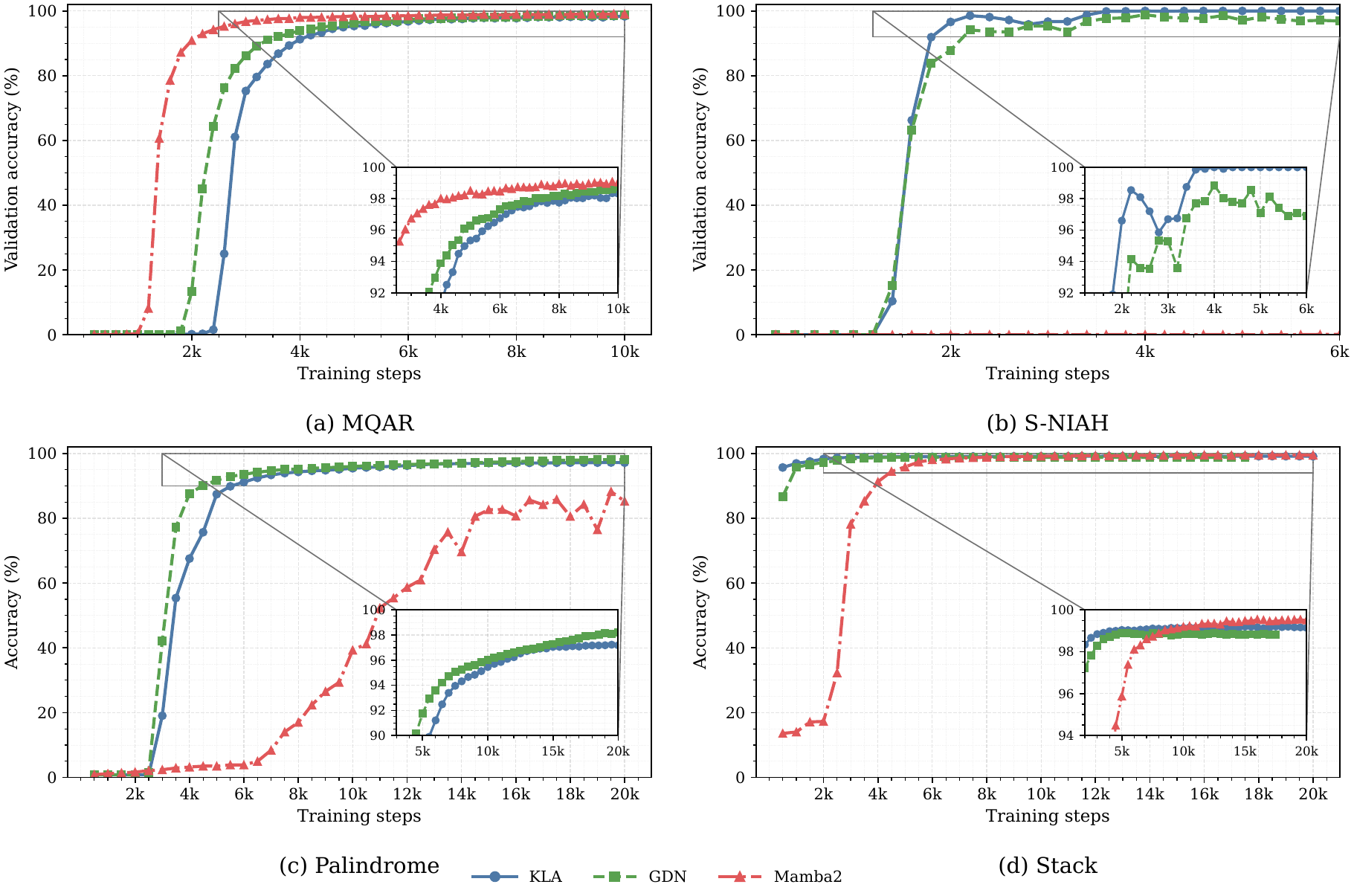}
\caption{Synthetic task training convergence. \Cref{tab:synthetic_extrap} reports length extrapolation. Insets zoom into the high-accuracy regime for each task.}
\label{fig:synthetic}
\end{figure*}

\subsection{Efficiency}
\label{sec:efficiency}

Efficiency tests whether the Kaczmarz coefficient preserves the runtime profile of the GDN chunkwise algorithm.
\method{} should leave prefill unchanged because it uses the same chunkwise triangular solve with a different diagonal coefficient matrix.

\paragraph{Prefill latency.}
\method{} matches GDN prefill latency.
\Cref{fig:efficiency}a shows that, at 131K context, prefill takes 1220~ms for \method{}, 1227~ms for GDN, and 1936~ms for Mamba2.
Computing $\norm{k_t}^2$ and substituting the Kaczmarz coefficient therefore do not change the dominant chunkwise kernel cost.

\paragraph{Decode throughput.}
\method{} improves decode throughput at long context.
At 32K context, \method{} reaches 3.54 tokens/s, which is $2.1\times$ higher than GDN (1.65 tokens/s) and $1.7\times$ higher than Mamba2 (2.10 tokens/s).
\Cref{fig:efficiency}b--c reports throughput and time per output token across context lengths, following standard decoding-efficiency metrics~\citep{shazeer2019fasttransformerdecodingwritehead,ainslie2023gqatraininggeneralizedmultiquery}.
The efficiency result supports the drop-in claim: \method{} changes the write coefficient but not the state layout, recurrent interface, or hardware kernel.

\begin{figure*}[t]
\centering
\begin{subfigure}[t]{0.32\textwidth}
    \centering
    \includegraphics[width=\linewidth]{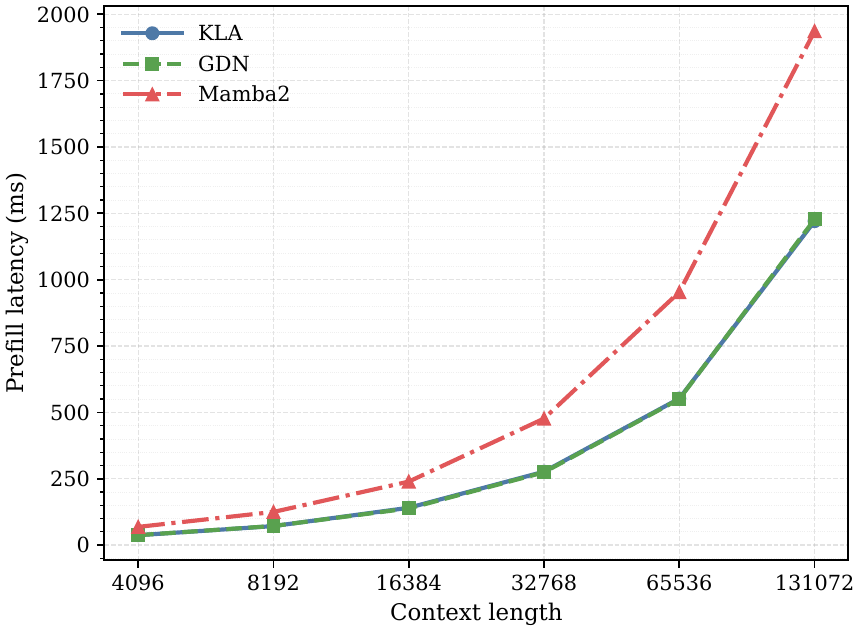}
    \caption{Prefill latency (ms) $\downarrow$}
\end{subfigure}
\hfill
\begin{subfigure}[t]{0.32\textwidth}
    \centering
    \includegraphics[width=\linewidth]{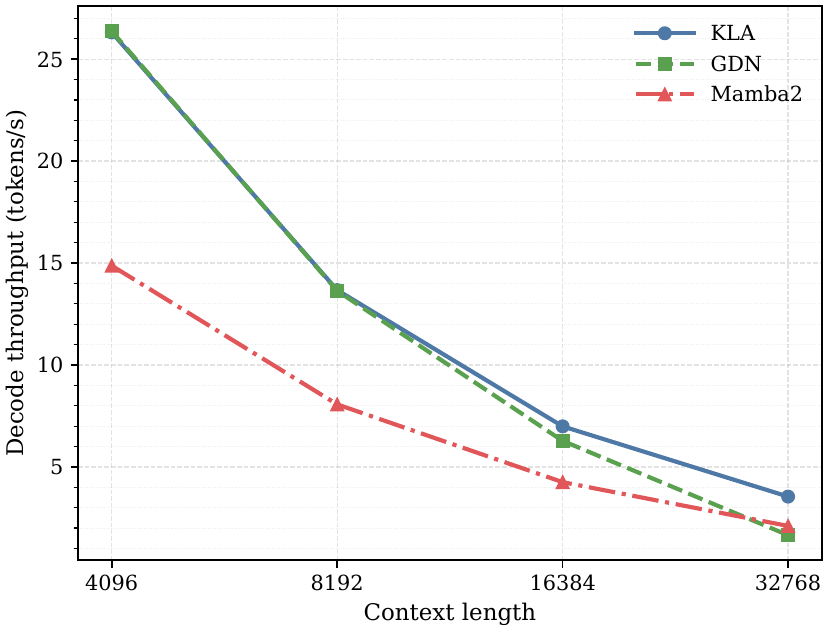}
    \caption{Decode throughput (tok/s) $\uparrow$}
\end{subfigure}
\hfill
\begin{subfigure}[t]{0.32\textwidth}
    \centering
    \includegraphics[width=\linewidth]{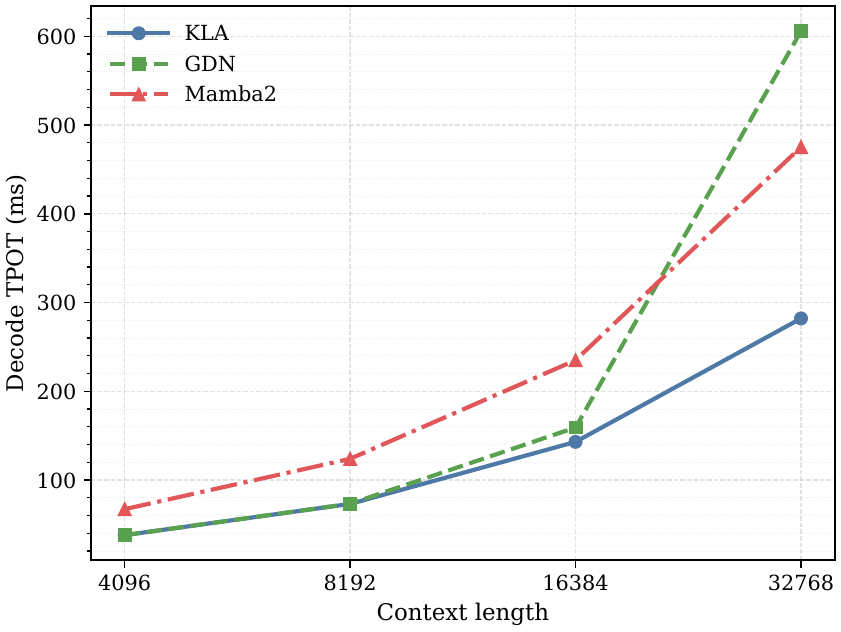}
    \caption{Decode time per output token (TPOT, ms/tok) $\downarrow$}
\end{subfigure}
\caption{Efficiency comparison with batch size 1. \method{} and GDN have nearly identical prefill profiles; \method{} achieves $2.1\times$ higher decode throughput than GDN at 32K context.}
\label{fig:efficiency}
\end{figure*}

\subsection{Ablation Studies}
\label{sec:ablation}

We ablate three design choices: the Kaczmarz normalization, the gating structure, and the value-state expansion ratio. Each variant is evaluated on MQAR at the training length and at $8\times$ extrapolation, and on a 100M-token SlimPajama pretraining run for language-modeling stability. Detailed hyperparameters and variant definitions are in \cref{app:ablation_setup}.

\paragraph{Normalization strategy.}
We compare the default coefficient $\beta_t=\eta_t/(\norm{k_t}^2+\eps)$ with no normalization, key-norm-only normalization, and an unconstrained learned scalar. \Cref{tab:ablations} shows that the full Kaczmarz coefficient gives the best stability trade-off: normalization is required for MQAR, while the learned scalar and key-norm-only variants do not match the pretraining stability of the default update.

\paragraph{Gating mechanism.}
We compare a single-gate variant, the dual-gate GDN reference, and an independent-gate GLA reference. The results test whether correcting the write coefficient is sufficient without the GDN-style gating structure.

\paragraph{State expansion.}
We vary the value expansion ratio $v_{\mathrm{expand}}\in\{2,4,8,16\}$ to test whether additional state capacity replaces coefficient correctness. Larger expansion can improve synthetic memorization, but it does not remove the language-modeling stability trade-off.

\begin{table}[H]
\centering
\small
\setlength{\tabcolsep}{4pt}
\renewcommand{\arraystretch}{1.05}
\begin{tabularx}{\linewidth}{l>{\raggedright\arraybackslash}Xccc}
\toprule
Group & Variant & Test Acc. (\%) & $8\times$ Acc. (\%) & Val. PPL @ 100M $\downarrow$ \\
\midrule
Normalization & \textbf{KLA (Default)} & \textbf{98.39} & 73.84 & \textbf{753.65} \\
Normalization & w/o Normalization & 0.03 & 0.01 & --- \\
Normalization & Key-norm only ($1/\norm{k_t}^2$) & 97.81 & \textbf{77.68} & 3809.71 \\
Normalization & Learned Scalar & 97.77 & 74.46 & 22968.87 \\
\addlinespace[0.25em]
Gating & Single Gate & 0.01 & 0.03 & 2104.61 \\
Gating & \textbf{Dual Gate (GDN reference)} & \textbf{98.72} & \textbf{66.81} & 747.24 \\
Gating & Independent Gate (GLA) & 0.03 & 0.03 & \textbf{356.01} \\
\addlinespace[0.25em]
Expansion & $v_{\mathrm{expand}}=2$ & 98.09 & 71.99 & \textbf{762.33} \\
Expansion & $v_{\mathrm{expand}}=4$ & 98.88 & 91.34 & 1669.35 \\
Expansion & $v_{\mathrm{expand}}=8$ & 95.31 & \textbf{95.37} & 6945.63 \\
Expansion & $v_{\mathrm{expand}}=16$ & \textbf{99.18} & 92.63 & 1383.38 \\
\bottomrule
\end{tabularx}
\caption{Ablation summary. Test and $8\times$ accuracy are measured on MQAR; validation PPL is measured after 100M pretraining tokens.}
\label{tab:ablations}
\end{table}

\paragraph{Related work.}
We provide the full related-work discussion in \cref{app:related}, covering linear attention and fast weights, state-space and hybrid sequence models, gated linear attention, delta-rule and test-time-learning methods, and efficient long-context attention systems.

\section{Limitations}
\label{sec:limitations}

\method{} improves the local write coefficient but retains the capacity limits of a fixed-size recurrent state.
The exact projection result in \cref{thm:projection} holds only when $\eta_t=1$ and $\eps=0$; the practical model trades exactness for numerical stability and adaptive write strength.
In low-precision arithmetic, very small key norms make the choice of $\eps$ important.
The longest Palindrome setting also shows a case where GDN slightly outperforms \method{}, suggesting that structured reversal may need mechanisms beyond key-norm-normalized residual writes.

\section{Conclusion}
\label{sec:conclusion}

\method{} shows that the delta-rule coefficient determines whether a recurrent write solves its local regression problem. Viewing the write as a constrained projection yields the Kaczmarz coefficient, a self-normalized replacement for GDN's learned scalar that becomes an orthogonal projection, exact line-search minimizer, and normalized gradient step in the unrelaxed setting.

Empirically, this coefficient improves 0.4B pretraining PPL by 0.41 over GDN, lowers long-context PPL by 3.7--4.1 across tested lengths, improves MQAR $8\times$ extrapolation by 7.03 points, reaches 100\% on S-NIAH, and preserves GDN-like prefill while improving 32K decode throughput by $2.1\times$. The ablations identify key-norm normalization as necessary for stable pretraining, while Palindrome remains a boundary case for normalized residual writes. Future work should test larger state layouts, diagonal or expanded gating, and larger-scale pretraining regimes.

\bibliographystyle{plainnat}
\bibliography{kla_refs}

\appendix

\section{Proofs for the Tokenwise Theory}
\label{app:proofs}

\subsection{Proof of \texorpdfstring{\Cref{thm:projection}}{Theorem 1}}

Let $\wt{S} = \wt{S}_{t-1}$, $e = e_t$, and $k = k_t$.

\paragraph{Part (ii): exact line search.}
Consider the one-parameter family $S(\tau) = \wt{S} + \tau k e\trans$. Substituting into the per-token loss:
\begin{align}
\Loss_t(S(\tau))
&= \frac{1}{2}\|S(\tau)\trans k - v_t\|_2^2
= \frac{1}{2}\bigl\|\wt{S}\trans k + \tau e\norm{k}^2 - v_t\bigr\|_2^2 \notag \\
&= \frac{1}{2}\bigl(-e + \tau\norm{k}^2 e\bigr)^2
= \frac{1}{2}\bigl(1 - \tau\norm{k}^2\bigr)^2\norm{e}^2.
\end{align}
Setting the derivative to zero: $\tau^* = 1/\norm{k}^2$.

\paragraph{Constraint satisfaction.}
Substituting $\tau^*$:
\begin{align}
S_t\trans k = \wt{S}\trans k + e\frac{k\trans k}{\norm{k}^2} = \wt{S}\trans k + e = v_t.
\end{align}
Hence $S_t \in \Aset_t$.

\paragraph{Part (i): orthogonal projection.}
The tangent space of $\Aset_t$ is $T_{\Aset_t} = \{H \in \R^{d_k \times d_v} : H\trans k = 0\}$. The correction $\Delta = S_t - \wt{S} = ke\trans/\norm{k}^2$. For any $H \in T_{\Aset_t}$:
\begin{align}
\langle \Delta, H \rangle_F
= \Tr\Bigl(\frac{1}{\norm{k}^2}ke\trans H\trans\Bigr)
= \frac{1}{\norm{k}^2} e\trans (H\trans k) = 0,
\end{align}
since $H\trans k = 0$. Therefore $\Delta \perp T_{\Aset_t}$, confirming that $S_t$ is the orthogonal projection of $\wt{S}$ onto $\Aset_t$.

\paragraph{Part (iii): normalized SGD.}
The gradient of $\Loss_t$ at $\wt{S}$ is $\nabla_S \Loss_t\big|_{S=\wt{S}} = k(\wt{S}\trans k - v_t)\trans = -k e\trans$. Therefore:
\[
S_t = \wt{S} - \frac{1}{\norm{k}^2}\nabla_S \Loss_t\big|_{S=\wt{S}},
\]
which is normalized gradient descent with coefficient $1/\norm{k}^2$.

\begin{remark}
By the projection theorem for Hilbert spaces, $S_t$ is the unique element of $\Aset_t$ satisfying $\fro{S_t - \wt{S}} \le \fro{S - \wt{S}}$ for all $S \in \Aset_t$, with equality only when $S = S_t$. The Kaczmarz update makes the minimum Frobenius-norm change to the state.
\end{remark}

\subsection{Alternative derivation via Lagrange multipliers}

The Lagrangian of~\eqref{eq:constrained} is $\mathcal{J}(S, \lambda) = \frac{1}{2}\fro{S - \wt{S}}^2 + \lambda\trans(S\trans k - v_t)$ with $\lambda \in \R^{d_v}$. Setting $\partial \mathcal{J}/\partial S = 0$:
\[
S - \wt{S} + k\lambda\trans = 0 \implies S = \wt{S} - k\lambda\trans.
\]
Substituting into $S\trans k = v_t$:
\begin{align}
(\wt{S} - k\lambda\trans)\trans k = v_t \implies \lambda = \frac{\wt{S}\trans k - v_t}{\norm{k}^2} = -\frac{e}{\norm{k}^2}.
\end{align}
Therefore $S_t = \wt{S} + ke\trans/\norm{k}^2$, recovering the Kaczmarz update. \qed

\subsection{Proof of \texorpdfstring{\Cref{prop:proximal}}{Proposition 1}}

The proximal problem is:
\[
\minimize_S \; \frac{1}{2}\fro{S-\wt{S}}^2 + \frac{\mu_t}{2}\|S\trans k - v_t\|_2^2.
\]
Setting the gradient to zero: $S - \wt{S} + \mu_t k(S\trans k - v_t)\trans = 0$.
Substituting $S = \wt{S} + \tau k e\trans$:
\begin{align}
\tau k e\trans + \mu_t k\bigl(\tau\norm{k}^2 e - e\bigr)\trans &= 0 \notag \\
k\,e\trans\bigl(\tau + \mu_t\tau\norm{k}^2 - \mu_t\bigr) &= 0.
\end{align}
Solving: $\tau = \mu_t/(1 + \mu_t\norm{k}^2)$. Substituting $\mu_t = \eta_t / ((1-\eta_t)\norm{k}^2 + \eps)$:
\begin{align}
\tau = \frac{\eta_t}{(1-\eta_t)\norm{k}^2+\eps+\eta_t\norm{k}^2} = \frac{\eta_t}{\norm{k}^2 + \eps},
\end{align}
which is exactly the \method{} coefficient. \qed

\subsection{Proof of \texorpdfstring{\Cref{prop:contraction}}{Proposition 2}}

Starting from~\eqref{eq:kla}:
\begin{align}
S_t\trans k
&= \wt{S}\trans k + \frac{\eta_t}{\norm{k}^2+\eps}e\,(k\trans k)
= \wt{S}\trans k + \eta_t\frac{\norm{k}^2}{\norm{k}^2+\eps}e.
\end{align}
Therefore:
\begin{align}
e_t^+
= v_t - S_t\trans k
= e - \eta_t\frac{\norm{k}^2}{\norm{k}^2+\eps}e
= \Bigl(1 - \eta_t\frac{\norm{k}^2}{\norm{k}^2+\eps}\Bigr)e.
\end{align}
For $\eta_t \in (0,1]$ and $\eps \ge 0$, the factor lies in $[0,1)$, so $\|e_t^+\|_2 < \|e_t\|_2$. Squaring both sides gives $\Loss_t(S_t) \le \Loss_t(\wt{S}_{t-1})$. \qed

\subsection{Decay-agnostic remark}

Let $D_t$ be any linear operator and $\wt{S} = D_t S_{t-1}$. The projection problem $\min_{S \in \Aset_t}\fro{S-\wt{S}}$ depends only on the base point $\wt{S}$ and the constraint $\Aset_t$. The derivation of \cref{thm:projection} uses only $e = v_t - \wt{S}\trans k$ and $\norm{k}^2$, making no assumption about how $\wt{S}$ was produced from $S_{t-1}$. The Kaczmarz coefficient applies regardless of whether $D_t = \alpha_t \I$, $D_t = \diag(\alpha_t)$, or any other linear decay.

\section{Derivations for Chunkwise Parallelism}
\label{app:chunk}

\Cref{tab:recurrent_parallel} summarizes the recurrent and chunkwise-parallel forms referenced in \cref{sec:chunk}.

\begin{table*}[t]
\centering
\scriptsize
\setlength{\tabcolsep}{3.2pt}
\renewcommand{\arraystretch}{1.12}
\resizebox{\textwidth}{!}{%
\begin{tabular}{lll}
\toprule
Method & Recurrent form & Chunkwise-parallel form \\
\midrule
Softmax attention & $\sum_{j\le t}\exp(q_t\trans k_j)\,v_j$ & $(\exp(QK\trans)\od M)V$ \\
Linear attention  & $\sum_{j\le t}(q_t\trans k_j)\,v_j$ & $(QK\trans\od M)V$ \\
RetNet            & $\sum_{j\le t}q_t\trans\!\bigl(\prod_{s=j+1}^{t}\alpha_s\bigr)k_j\,v_j$ & $(QK\trans\od A\od M)V$ \\
Mamba2            & $\sum_{j\le t}q_t\trans\!\bigl(\prod_{s=j+1}^{t}\alpha_s\bigr)k_j\,v_j$ & $(QK\trans\od A\od M)V$ \\
GLA               & $\sum_{j\le t}q_t\trans\!\bigl(\prod_{s=j+1}^{t}\diag(\alpha_s)\bigr)k_j\,v_j$ & $((Q\od\Gamma)(K/\Gamma)\trans\od M)V$ \\
DeltaNet          & $\sum_{j\le t}q_t\trans\!\bigl(\prod_{s=j+1}^{t}(\I-\beta_s k_s k_s\trans)\bigr)k_j\,v_j$ & $(QK\trans\od M)\,\mathcal{T}_\beta(V)$ \\
GDN               & $\sum_{j\le t}q_t\trans\!\bigl(\prod_{s=j+1}^{t}\alpha_s(\I-\beta_s k_s k_s\trans)\bigr)k_j\,v_j$ & $(QK\trans\od A\od M)\,\mathcal{T}_\beta(V)$ \\
\rowhi \textbf{KLA} & $\sum_{j\le t}q_t\trans\!\bigl(\prod_{s=j+1}^{t}\alpha_s(\I-\kmark{\beta_s} k_s k_s\trans)\bigr)k_j\,v_j$ & $(QK\trans\od A\od M)\,\mathcal{T}_{\kmark{\mathrm{kacz}}}(V)$ \\
\bottomrule
\end{tabular}}
\caption{Recurrent and chunkwise-parallel forms. $M$ is the causal mask; $A$ and $\Gamma$ are accumulated decay factors. \method{} has the same recurrent form as GDN; only the coefficient $\beta_t = \kmark{\eta_t/(\norm{k_t}^2+\eps)}$ differs.}
\label{tab:recurrent_parallel}
\end{table*}

We first recall the chunk-level recurrence. Consider a chunk of length $C$ with local indices $i = 1,\ldots,C$. The incoming transposed state is $R_0 = S_0\trans \in \R^{d_v \times d_k}$. From \cref{eq:transposed}:
\begin{equation*}
\begin{aligned}
R_i &= \underbrace{\Bigl(\prod_{r=1}^{i}\alpha_r(\I - \beta_r k_r k_r\trans)\Bigr)}_{:=\,P_i}\cdot R_0
\;+\;\underbrace{\sum_{j=1}^{i}\Bigl(\prod_{r=j+1}^{i}\alpha_r(\I - \beta_r k_r k_r\trans)\Bigr)\cdot\beta_j v_j k_j\trans}_{:=\,H_i}
\\[4pt]
&= P_i \cdot R_0 + H_i
\end{aligned}
\end{equation*}
where $\beta_i = \eta_i/(\norm{k_i}^2 + \eps)$. Our goal is to transform $P_i$ and $H_i$ into forms suitable for parallel computation.

Define the cumulative decay $\gamma_i = \prod_{r=1}^{i}\alpha_r$ ($\gamma_0 = 1$) and the cumulative decay from position $j$ to position $i$:
\[
\gamma_{j \to i} = \prod_{r=j}^{i}\alpha_r = \frac{\gamma_i}{\gamma_{j-1}}.
\]

We show that $P_i$, a product of scalar-decayed Householder-like matrices, admits a WY representation.

\begin{proposition}[WY representation of $P_i$]
\label{prop:wy_P}
The matrix $P_i$ can be expressed as:
\begin{equation}
P_i = \gamma_i \I - \sum_{r=1}^{i}\gamma_{r \to i}\, k_r\, w_r\trans
\label{eq:P_wy}
\end{equation}
where the auxiliary vector $w_r \in \R^{d_k}$ is computed via the recurrence:
\begin{equation}
w_i = \beta_i\Bigl(\gamma_{i \to i}\, k_i - \sum_{r=1}^{i-1}w_r\bigl(k_r\trans \gamma_{r \to i}\, k_i\bigr)\Bigr).
\label{eq:w_recursion}
\end{equation}
\end{proposition}

\begin{proof}
We proceed by induction on $i$.

\textbf{Inductive step.}
Assume $P_{i-1} = \gamma_{i-1}\I - \sum_{r=1}^{i-1}\gamma_{r \to i-1}\,k_r\,w_r\trans$.
We derive:
\begin{align*}
P_i
&= \alpha_i(\I - \beta_i k_i k_i\trans)\; P_{i-1} \\
&= \alpha_i(\I - \beta_i k_i k_i\trans)\;\Bigl(\gamma_{i-1}\I - \sum_{r=1}^{i-1}\gamma_{r \to i-1}\,k_r\,w_r\trans\Bigr) \\
&= \alpha_i\Bigl(\gamma_{i-1}\I - \sum_{r=1}^{i-1}\gamma_{r \to i-1}\,k_r\,w_r\trans\Bigr) - \alpha_i\beta_i k_i k_i\trans\Bigl(\gamma_{i-1}\I - \sum_{r=1}^{i-1}\gamma_{r \to i-1}\,k_r\,w_r\trans\Bigr) \\
&= \gamma_i \I - \sum_{r=1}^{i-1}\gamma_{r \to i}\,k_r\,w_r\trans
   - \beta_i k_i\Bigl(\gamma_{i \to i}\, k_i - \sum_{r=1}^{i-1}w_r\,(k_r\trans \gamma_{r \to i}\, k_i)\Bigr)\trans \\
&= \gamma_i \I - \sum_{r=1}^{i-1}\gamma_{r \to i}\,k_r\,w_r\trans
   - k_i\;\underbrace{\beta_i\Bigl(\gamma_{i \to i}\, k_i - \sum_{r=1}^{i-1}w_r\,(k_r\trans\gamma_{r \to i}\, k_i)\Bigr)\trans}_{w_i\trans} \\
&= \gamma_i \I - \sum_{r=1}^{i-1}\gamma_{r \to i}\,k_r\,w_r\trans - k_i\,w_i\trans \\
&= \gamma_i \I - \sum_{r=1}^{i}\gamma_{r \to i}\,k_r\,w_r\trans
\end{align*}
where in the fourth line we used $\alpha_i\gamma_{i-1} = \gamma_i$ and $\alpha_i\gamma_{r \to i-1} = \gamma_{r \to i}$, and in the fifth line we factored out $k_i$ and identified the underbraced expression as~\eqref{eq:w_recursion}.
\end{proof}

Similarly, $H_i$ can be expressed in a parallelizable form.

\begin{proposition}[WY representation of $H_i$]
\label{prop:wy_H}
The matrix $H_i$ can be expressed as:
\begin{equation}
H_i = \sum_{r=1}^{i}\gamma_{r \to i}\, k_r\, u_r\trans
\label{eq:H_wy}
\end{equation}
where the auxiliary vector $u_r \in \R^{d_v}$ is computed via the recurrence:
\begin{equation}
u_i = \beta_i\Bigl(v_i - \sum_{r=1}^{i-1}u_r\bigl(k_r\trans \gamma_{r \to i}\, k_i\bigr)\Bigr).
\label{eq:u_recursion}
\end{equation}
\end{proposition}

\begin{proof}
We again use induction on $i$.

\textbf{Inductive step.}
Assume $H_{i-1} = \sum_{r=1}^{i-1}\gamma_{r \to i-1}\,k_r\,u_r\trans$.
\begin{align*}
H_i
&= \alpha_i(\I - \beta_i k_i k_i\trans)\;H_{i-1} + \beta_i v_i k_i\trans \\
&= \alpha_i(\I - \beta_i k_i k_i\trans)\;\Bigl(\sum_{r=1}^{i-1}\gamma_{r \to i-1}\,k_r\,u_r\trans\Bigr) + \beta_i v_i k_i\trans \\
&= \alpha_i \sum_{r=1}^{i-1}\gamma_{r \to i-1}\,k_r\,u_r\trans
   - \alpha_i\beta_i k_i k_i\trans \sum_{r=1}^{i-1}\gamma_{r \to i-1}\,k_r\,u_r\trans
   + \beta_i v_i k_i\trans \\
&= \sum_{r=1}^{i-1}\gamma_{r \to i}\,k_r\,u_r\trans
   - \beta_i k_i \Bigl(\sum_{r=1}^{i-1}\bigl(k_r\trans\gamma_{r \to i}\, k_i\bigr)\,u_r\Bigr)\trans
   + \beta_i v_i k_i\trans \\
&= \sum_{r=1}^{i-1}\gamma_{r \to i}\,k_r\,u_r\trans
   + k_i\;\underbrace{\beta_i\Bigl(v_i - \sum_{r=1}^{i-1}u_r\bigl(k_r\trans\gamma_{r \to i}\, k_i\bigr)\Bigr)\trans}_{u_i\trans} \\
&= \sum_{r=1}^{i-1}\gamma_{r \to i}\,k_r\,u_r\trans + k_i\,u_i\trans \\
&= \sum_{r=1}^{i}\gamma_{r \to i}\,k_r\,u_r\trans
\end{align*}
where in the fourth line we used $\alpha_i\gamma_{r \to i-1} = \gamma_{r \to i}$ and rearranged transposes, and in the fifth line we identified the underbraced expression as~\eqref{eq:u_recursion}.
\end{proof}

\subsection{Combined WY representation}

Combining Propositions~\ref{prop:wy_P} and~\ref{prop:wy_H}, the state at position $i$ is:
\begin{equation}
R_i = P_i \cdot R_0 + H_i = \gamma_i R_0 - \sum_{r=1}^{i}\gamma_{r \to i}\,k_r\,(R_0\trans w_r)\trans + \sum_{r=1}^{i}\gamma_{r \to i}\,k_r\, u_r\trans = \gamma_i R_0 + \sum_{r=1}^{i}\gamma_{r \to i}\,k_r\,\hat{u}_r\trans,
\label{eq:combined_wy}
\end{equation}
where $\hat{u}_r = u_r - R_0\trans w_r$ combines both auxiliary vectors. In practice, it is more efficient to absorb the $R_0$ terms directly into the $u$ recurrence. Substituting $\gamma_{r \to i} = \gamma_i/\gamma_r$ when $\gamma_r > 0$ and defining $\bar{u}_r := \hat{u}_r$ gives the representation underlying the chunkwise system in \cref{prop:chunk}.

\subsection{Matrix form (\texorpdfstring{\Cref{prop:chunk}}{Proposition 3} proof)}

Stack all vectors row-wise:
\[
K = \begin{bmatrix}k_1\trans\\\vdots\\k_C\trans\end{bmatrix}\in\R^{C\times d_k},\quad
V = \begin{bmatrix}v_1\trans\\\vdots\\v_C\trans\end{bmatrix}\in\R^{C\times d_v},\quad
U = \begin{bmatrix}u_1\trans\\\vdots\\u_C\trans\end{bmatrix}\in\R^{C\times d_v}.
\]
The Gram matrix is $G = KK\trans \in \R^{C \times C}$ with $G_{ij} = k_i\trans k_j$. For the row corresponding to position $i$, the incoming-state prediction is $(\gamma_i R_0 k_i)\trans = (D_\gamma KR_0\trans)_i$, where $D_\gamma = \diag(\gamma_1,\ldots,\gamma_C)$.

Define the strict lower-triangular matrix:
\[
L_{ij} = \begin{cases} \beta_i\,\dfrac{\gamma_i}{\gamma_j}\,G_{ij}, & j < i, \\ 0, & j \ge i. \end{cases}
\]
Using the causal decay matrix $A^-_{ij} = \gamma_i/\gamma_j$ for $j < i$ (zero otherwise), we have $L = B(A^- \od G)$ with $B = \diag(\beta_1,\ldots,\beta_C)$.

The $u$-recurrence~\eqref{eq:u_recursion} (augmented with the $R_0$ terms from~\eqref{eq:combined_wy}) in matrix form reads:
\[
U = B(V - D_\gamma KR_0\trans) - LU,
\]
which rearranges to:
\[
(\I + L)\,U = B(V - D_\gamma KR_0\trans).
\]
Since $R_0 = S_0\trans$ and $L = B(A^- \od KK\trans)$, this is \cref{eq:chunk_system}. The matrix $\I + L$ is unit lower-triangular (its diagonal entries are all 1, since $L$ is strictly lower-triangular), so it is always nonsingular and is solved by forward substitution. \qed

\subsection{Chunk outputs and outgoing state (\texorpdfstring{\Cref{prop:chunk_output}}{Proposition 4} proof)}

\paragraph{Token outputs.}
The output at position $i$ is $o_i = R_i q_i$. Substituting the combined WY representation~\eqref{eq:combined_wy}:
\begin{align*}
o_i &= \gamma_i R_0 q_i + \sum_{j=1}^{i}\frac{\gamma_i}{\gamma_j}\,u_j\,(k_j\trans q_i).
\end{align*}
Stacking all $i = 1,\ldots,C$:
\[
O = D_\gamma Q S_0 + (A \od QK\trans)\,U,
\]
where $A_{ij} = \gamma_i/\gamma_j$ for $j \le i$ (the full lower-triangular version including the diagonal). This is \cref{eq:chunk_out}. The first term is the readout from the incoming state after decay; the second is the contribution of updates within the chunk.

\paragraph{Outgoing state.}
Setting $i = C$ in the combined representation and transposing back:
\begin{align*}
S_{\mathrm{out}} = R_C\trans
&= \Bigl(\gamma_C R_0 + \sum_{j=1}^{C}\frac{\gamma_C}{\gamma_j}\,u_j\, k_j\trans\Bigr)\trans \\
&= \gamma_C S_0 + \sum_{j=1}^{C}\frac{\gamma_C}{\gamma_j}\,k_j\, u_j\trans \\
&= \gamma_C S_0 + K\trans \diag\!\Bigl(\frac{\gamma_C}{\gamma_1},\ldots,\frac{\gamma_C}{\gamma_C}\Bigr)U,
\end{align*}
which is \cref{eq:chunk_state}. \qed

\subsection{Why coefficient substitution suffices}

The entire derivation uses $\beta_i$ only through the diagonal matrix $B = \diag(\beta_1,\ldots,\beta_C)$ and the recurrences~\eqref{eq:w_recursion}--\eqref{eq:u_recursion}. No structural property of the GDN coefficient---such as the absence of $\norm{k_i}$ in the denominator---is invoked at any step. Replacing $\beta_i = \eta_i$ with $\beta_i = \eta_i/(\norm{k_i}^2 + \eps)$ therefore produces the complete \method{} chunkwise algorithm by substituting a different diagonal matrix $B$ into the same GDN kernel. No new CUDA kernel or parallel reduction is required.

\subsection{Extension to diagonal gating}

When the decay operator is diagonal, $D_t = \diag(\alpha_t)$ (as in GLA and KDA~\citep{kimiteam2025kimilinearexpressiveefficient}), the scalar $\gamma_i$ becomes a vector $\bm{\gamma}_i$ and all decay products become elementwise. The WY propositions generalize by replacing $\gamma_{r \to i}$ with $\diag(\bm{\gamma}_{r \to i})$ wherever it appears. The local Kaczmarz coefficient still applies unchanged; only the transition algebra inside the chunk changes. The chunkwise derivation for diagonal-gated \method{} follows the KDA template with $B$ substituted accordingly.

\section{Experimental Details}
\label{app:details}

\paragraph{Architecture.}
All models use 24 layers, hidden size 1024, 8 attention heads, head dimension 128, and FFN expansion $4\times$ (0.4B parameters total).

\paragraph{Pretraining hyperparameters.}
AdamW with $\beta_1=0.9$, $\beta_2=0.95$, weight decay $0.1$, gradient clipping $1.0$, cosine learning rate schedule with linear warmup over 2\% of training steps. Stabilizer $\eps = 10^{-6}$.

\paragraph{Compute resources.}
All reported experiments used 4 NVIDIA A100 80GB GPUs. Each 1B-token pretraining run took approximately 4 hours. We did not run multi-seed pretraining because of limited compute.

\paragraph{Synthetic tasks.}
All tasks follow the GDN~\citep{yang2025gated} configurations with seed~42. MQAR: 2-layer model, 10K steps, trained at length~256. S-NIAH: 6K steps. Palindrome: 20K steps. Stack: 20K steps.

\section{Ablation Study Experimental Setup}
\label{app:ablation_setup}

For the ablations in Section~\ref{sec:ablation}, each architectural variant uses the same task pipeline and random seed ($42$). 

\paragraph{MQAR Evaluation Protocol.}
The Multi-Query Associative Recall (MQAR) datasets use sequence length $256$, key length $1$, and $32$ key-value pairs.
Keys are sampled from the first half of the vocabulary, and values from the second half.
The split contains 20,000 training, 2,000 validation, and 2,000 test sequences.
All variants train for at most 10,000 steps with batch size $32$, learning rate $10^{-3}$, and weight decay $0.1$.
Early stopping uses a patience of $10$ evaluations, with evaluation every 200 steps.
Length extrapolation is evaluated at test factors of $1\times$, $2\times$, $4\times$, and $8\times$, prioritizing the $8\times$ metric for final reporting.

\paragraph{100M-Token Pretraining Protocol.}
To evaluate general language modeling stability, each variant undergoes pretraining on a 100M-token subset of the SlimPajama training split.
We use sequence length $4096$ and an optimizer learning rate of $10^{-4}$.
Validation perplexity is computed over $20$ validation runs, each averaging $15$ iterations.
For high-capacity ablations such as $v_{\text{expand}} \in \{8, 16\}$, the micro-batch size is adjusted downward to avoid out-of-memory errors without changing the global token budget.

\paragraph{Variant Configurations.}
The ablation baseline uses the default KLA configuration: $\beta_t = \eta_t/(\|k_t\|_2^2 + \varepsilon)$ with dual gating.
\begin{itemize}[leftmargin=1.5em]
    \item \textbf{Ablation 1 (Normalization):} Changes the coefficient $\beta_t$. The ``Learned Scalar'' initializes a trainable parameter at $1.0$. ``NoNorm'' reverts to a standard sigmoid gating $\beta_t = \sigma(b_t)$.
    \item \textbf{Ablation 2 (Sequence Factor):} Multiplies the KLA coefficient by $1/t$, $1/\sqrt{t}$, or $1/\log(t+1)$, where $t$ is the absolute sequence index.
    \item \textbf{Ablation 2 (Gating):} Changes the recurrent decay and update gates. ``Single Gate'' binds the decay and delta-update to a single projected mechanism.
    \item \textbf{Ablation 3 (State Dimension):} Changes the internal value expansion ratio in the linear-attention projection block.
\end{itemize}

\section{Related Work}
\label{app:related}

\paragraph{Positioning.}
We compare sub-quadratic sequence models by the part of the long-context mechanism they change: the attention approximation, recurrent transition, gate, state layout, or online write rule.
\Cref{tab:positioning} localizes \method{} in this design space.
The claim of this paper is intentionally narrow: given the scalar-gated delta-rule template of GDN, the write coefficient should be the closed-form Kaczmarz step size rather than a learned or heuristic scalar.
Thus, \method{} keeps the GDN state, gates, recurrence, and chunkwise solver, but changes the coefficient that determines whether each residual write solves its local regression problem.

\begin{table*}[t]
\centering
\small
\renewcommand{\arraystretch}{1.2}
\resizebox{\textwidth}{!}{%
\begin{tabular}{lcccc}
\toprule
Model & Update mechanism & Gating granularity & Normalization / step size & Theoretical basis \\
\midrule
Mamba2 \citep{dao2024transformersssmsgeneralizedmodels} & Additive & Scalar (decay) & None & Continuous-time SSM \\
GLA \citep{pmlr-v235-yang24ab} & Additive & Diagonal (decay) & None & Fast weights \\
KDA \citep{kimiteam2025kimilinearexpressiveefficient} & Delta rule & Diagonal (decay) & Local $\ell_2$ + exponentiation & Heuristic scaling \\
GDN \citep{yang2025gated} & Delta rule & Scalar (decay + update) & Learned scalar $\eta_t$ & Heuristic design \\
\rowcolor{gray!10} \textbf{\method{} (Ours)} & Delta rule & Scalar (decay + update) & Kaczmarz $\eta_t/(\|k_t\|_2^2 + \varepsilon)$ & Per-token optimization \\
\bottomrule
\end{tabular}}
\caption{Positioning of \method{} against representative sub-quadratic sequence models. \method{} is the only model in the table that derives the key-norm normalization in its delta-rule step from exact per-token loss minimization.}
\label{tab:positioning}
\end{table*}

\paragraph{Linear attention, RNNs, and fast weights.}
Linear attention replaces softmax attention with kernelized inner products and exposes an equivalent recurrent form with a fixed-size associative state~\citep{pmlr-v119-katharopoulos20a}.
Fast-weight interpretations make the same object explicit: the state is a rapidly updated matrix written by the sequence itself~\citep{pmlr-v139-schlag21a,6796337}.
RetNet, RWKV, and Attention-Free Transformer refine retention or channel mixing, while xLSTM, HGRN2, and GateLoop add extended memory, state expansion, or data-controlled transitions~\citep{sun2023retentivenetworksuccessortransformer,peng2023rwkvreinventingrnnstransformer,zhai2021attentionfreetransformer,beck2024xlstmextendedlongshortterm,qin2024hgrn,katsch2024gateloopfullydatacontrolledlinear}.
These models establish the value of fixed-size recurrent memory, but they do not isolate the delta-rule step length as an optimization target.
\method{} follows the associative-state view, yet it does not introduce a new recurrent cell or a larger state; it isolates the rank-one residual write coefficient and normalizes it by the current key norm.

\paragraph{State-space and hybrid sequence models.}
State-space models pursue linear-time sequence modeling through structured latent dynamics rather than matrix-valued associative writes.
S4, H3, and S5 use structured state transitions or simplified state-space layers~\citep{gu2022efficientlymodelinglongsequences,fu2023hungryhungryhipposlanguage,smith2023simplified}.
Mamba and Mamba2 scale selective input-dependent dynamics and structured state-space duality~\citep{gu2024mamba,dao2024transformersssmsgeneralizedmodels,waleffe2024empiricalstudymambabasedlanguage}.
Hybrid systems such as Griffin and Jamba combine gated linear recurrences or Mamba-style layers with local or full attention to recover attention-like behavior~\citep{de2024griffinmixinggatedlinear,lieber2024jambahybridtransformermambalanguage}.
These works optimize the latent transition or the mixture of recurrent and attention layers.
\method{} addresses a different object: a matrix-state delta-rule update whose content selectivity comes from projecting along the current key direction.

\paragraph{Gated linear attention and state allocation.}
Gated linear attention adds data-dependent decay to linear attention while preserving hardware-efficient training~\citep{pmlr-v235-yang24ab}.
Subsequent variants change how the recurrent state is allocated, decayed, or parallelized.
Examples include constant-speed Lightning Attention, slot-based gated state, log-linear attention, and Kimi Linear's KDA with diagonal gating and a specialized chunkwise solver~\citep{qin2024variouslengthsconstantspeed,zhang2024gatedslotattentionefficient,guo2026loglinearattention,kimiteam2025kimilinearexpressiveefficient}.
This line of work mainly asks how expressive the gate and state layout should be.
\method{} asks an orthogonal question: for a fixed GDN-style scalar gate and state shape, what coefficient makes the residual write an exact local projection in the unrelaxed limit?

\paragraph{Delta-rule models and test-time learning.}
Delta-rule models replace additive fast-weight writes with residual-correcting updates.
DeltaNet parallelizes these updates over sequence length, and GDN adds scalar decay and update gates; \cref{sec:background} gives their formulations because they are the direct predecessors of \method{}~\citep{yang2024parallelizing,yang2025gated}.
Related online-learning views include Longhorn's amortized online learner and DeltaProduct's Householder-product state-tracking mechanism~\citep{liu2024longhornstatespacemodels,siems2025deltaproductimprovingstatetrackinglinear}.
Broader test-time-learning methods adapt hidden states or weights during inference through expressive inner learners, regression views, forgetting mechanisms, or test-time-training objectives~\citep{sun2025learninglearntesttime,wang2025testtimeregressionunifyingframework,lin2025forgettingtransformersoftmaxattention,zhang2025testtimetrainingright}.
\method{} is closest to GDN, but the modification is smaller and more specific than these broader online learners: it adds no auxiliary model, no iterative inner loop, and no new state type.
Its contribution is to replace the learned scalar step with the Kaczmarz coefficient for the single-constraint regression problem.

\paragraph{Efficient attention and long-context systems.}
Efficient-attention methods reduce the cost of softmax attention through IO-aware kernels, work partitioning, sparse patterns, low-rank or kernelized approximations, blockwise processing, and ring-style sequence parallelism~\citep{dao2022flashattentionfastmemoryefficientexact,dao2023flashattention2fasterattentionbetter,kitaev2020reformerefficienttransformer,zaheer2021bigbirdtransformerslonger,wang2020linformerselfattentionlinearcomplexity,choromanski2022rethinkingattentionperformers,liu2023blockwiseparalleltransformerlarge,brandon2023stripedattentionfasterring}.
Decoding-oriented variants such as multi-query and grouped-query attention reduce key-value cache size or head redundancy while keeping the attention interface~\citep{shazeer2019fasttransformerdecodingwritehead,ainslie2023gqatraininggeneralizedmultiquery}.
These methods make token-token attention cheaper, sparser, lower-rank, or more parallel.
\method{} removes the token-token attention matrix during autoregressive generation by using a fixed-size recurrent state, and its Kaczmarz coefficient fits into the existing GDN chunkwise algorithm without changing the kernel structure.


\end{document}